%% file: main.tex
\documentclass{article}

    \usepackage[preprint,nonatbib]{neurips_2020}

\usepackage[utf8]{inputenc} 
\usepackage[T1]{fontenc}    
\usepackage[colorlinks=true,allcolors=green,pagebackref=true]{hyperref}
\usepackage{url}            
\usepackage{booktabs}       
\usepackage{amsfonts}       
\usepackage{nicefrac}       
\usepackage{microtype}      
\usepackage{amsthm}

\input{preamble_packages}

\input{preamble_symbols}

\input{shortcuts}

\title{A Dynamical Perspective on Point Cloud Registration}

\author{%
  Heng Yang\\
  Laboratory for Information and Decision Systems (LIDS) \\
  Massachusetts Institute of Technology \\
  \texttt{hankyang@mit.edu}
}

\begin{document}
\maketitle

\input{abstract}
\input{intro}

\input{dynamicalConstruction}
\input{stabilityAnalysis}
\input{experiments}
\input{conclusions}

\input{acknowledgements}



{\small
\bibliographystyle{ieee_fullname}
\bibliography{../../references/refs.bib,myRefs.bib}
}

\end{document}

%% file: preamble_packages.tex
\usepackage{comment}
\usepackage{siunitx}
\usepackage{relsize}
\usepackage{ifthen}
\usepackage[colorinlistoftodos]{todonotes}

\usepackage[caption=false]{subfig}

\usepackage[vlined,ruled,linesnumbered]{algorithm2e}
\usepackage{graphics} 
\usepackage{rotating}
\usepackage{color}
\usepackage{enumerate}
\usepackage[T1]{fontenc}
\usepackage{psfrag}
\usepackage{epsfig} 
\usepackage{booktabs}
\usepackage{graphicx,url}
\usepackage{multirow}
\usepackage{array}
\usepackage{latexsym}
\usepackage{amsfonts}
\usepackage{amsmath}
\usepackage{amssymb}
\usepackage{xstring}
\usepackage[noend]{algorithmic}
\usepackage{multirow}
\usepackage{xcolor}
\usepackage{prettyref}
\usepackage{flexisym}
\usepackage{bigdelim}
\usepackage{breqn} 
\usepackage{listings}

\usepackage{enumitem}
\usepackage{xspace}
\usepackage{bm}
\graphicspath{{./figures/}}
\usepackage{tikz}
\usetikzlibrary{matrix,calc}

\usepackage{mdwlist}

\makecompactlist{itemize}{stditemize}


%% file: preamble_symbols.tex
\newrefformat{prob}{Problem\,\ref{#1}}
\newrefformat{def}{Definition\,\ref{#1}}
\newrefformat{sec}{Section\,\ref{#1}}
\newrefformat{sub}{Section\,\ref{#1}}
\newrefformat{prop}{Proposition\,\ref{#1}}
\newrefformat{app}{Appendix\,\ref{#1}}
\newrefformat{alg}{Algorithm\,\ref{#1}}
\newrefformat{cor}{Corollary\,\ref{#1}}
\newrefformat{thm}{Theorem\,\ref{#1}}
\newrefformat{lem}{Lemma\,\ref{#1}}
\newrefformat{fig}{Fig.\,\ref{#1}}
\newrefformat{tab}{Table\,\ref{#1}}

\newtheorem{theorem}{Theorem}

\newtheorem{corollary}[theorem]{Corollary}
\newtheorem{conjecture}[theorem]{Conjecture}

\newtheorem{definition}[theorem]{Definition}
\newtheorem{proposition}[theorem]{Proposition}

\newcommand{\bdmath}{\begin{dmath}}
\newcommand{\edmath}{\end{dmath}}
\newcommand{\beq}{\begin{equation}}
\newcommand{\eeq}{\end{equation}}
\newcommand{\bdm}{\begin{displaymath}}
\newcommand{\edm}{\end{displaymath}}
\newcommand{\bea}{\begin{eqnarray}}
\newcommand{\eea}{\end{eqnarray}}
\newcommand{\beal}{\beq \begin{array}{ll}}
\newcommand{\eeal}{\end{array} \eeq}
\newcommand{\beas}{\begin{eqnarray*}}
\newcommand{\eeas}{\end{eqnarray*}}
\newcommand{\ba}{\begin{array}}
\newcommand{\ea}{\end{array}}
\newcommand{\bit}{\begin{itemize}}
\newcommand{\eit}{\end{itemize}}
\newcommand{\ben}{\begin{enumerate}}
\newcommand{\een}{\end{enumerate}}

\newcommand{\calA}{{\cal A}}

\newcommand{\calE}{{\cal E}}
\newcommand{\calF}{{\cal F}}

\newcommand{\calI}{{\cal I}}

\newcommand{\calN}{{\cal N}}

\newcommand{\calR}{{\cal R}}

\newcommand{\calX}{{\cal X}}

\newcommand{\etal}{\emph{et~al.}\xspace}

\newcommand{\M}[1]{{\bm #1}} 
\renewcommand{\boldsymbol}[1]{{\bm #1}}

\newcommand{\hide}[1]{}
\newcommand{\wrt}{w.r.t.\xspace}

\newcommand{\hiddenText}{{\color{gray} hidden text.}}
\newcommand{\hideWithText}[1]{\hiddenText}

\DeclareMathOperator*{\argmin}{arg\,min}

\newcommand{\tran}{^{\mathsf{T}}}

\newcommand{\trace}[1]{\mathrm{tr}\left(#1\right)}

\newcommand{\zero}{{\mathbf 0}}
\newcommand{\eye}{{\mathbf I}}

\newcommand{\Real}[1]{ { {\mathbb R}^{#1} } }

\newcommand{\SOthree}{\ensuremath{\mathrm{SO}(3)}\xspace}

\newcommand{\MA}{\M{A}}

\newcommand{\MJ}{\M{J}}

\newcommand{\MR}{\M{R}}

\newcommand{\va}{\boldsymbol{a}}

\newcommand{\vf}{\boldsymbol{f}}

\newcommand{\vr}{\boldsymbol{r}}
\newcommand{\vs}{\boldsymbol{s}}

\newcommand{\vv}{\boldsymbol{v}}
\newcommand{\vt}{\boldsymbol{t}}
\newcommand{\vxx}{\boldsymbol{x}} 
\newcommand{\vy}{\boldsymbol{y}}

\newcommand{\valpha}{\boldsymbol{\alpha}}

\newcommand{\vepsilon}{\boldsymbol{\epsilon}}
\newcommand{\vtau}{\boldsymbol{\tau}}

\newcommand{\blue}[1]{{\color{blue}#1}}

\newcommand{\red}[1]{{\color{red}#1}}

\newcommand{\linkToPdf}[1]{\href{#1}{\blue{(pdf)}}}
\newcommand{\linkToPpt}[1]{\href{#1}{\blue{(ppt)}}}
\newcommand{\linkToCode}[1]{\href{#1}{\blue{(code)}}}
\newcommand{\linkToWeb}[1]{\href{#1}{\blue{(web)}}}
\newcommand{\linkToVideo}[1]{\href{#1}{\blue{(video)}}}
\newcommand{\linkToMedia}[1]{\href{#1}{\blue{(media)}}}
\newcommand{\award}[1]{\xspace} 



%% file: shortcuts.tex
\newcommand{\calY}{\mathcal{Y}}

\newcommand{\eg}{\emph{e.g.,}\xspace}
\newcommand{\ie}{\emph{i.e.,}\xspace}

\newcommand{\hatmap}[1]{\left[ #1 \right]_{\times}}

\newcommand{\half}{\frac{1}{2}}
\newcommand{\vectorize}[1]{\text{vec}\left(#1 \right)}

\newcommand{\xcom}{\bar{\vxx}}

\newcommand{\dxcom}{\dot{\xcom}}
\newcommand{\xcomv}{\bar{\vv}}
\newcommand{\xcoma}{\bar{\va}}

\newcommand{\xangv}{\boldsymbol{\omega}}
\newcommand{\xanga}{\valpha}
\newcommand{\sumallpoints}{\sum_{i=1}^N}
\newcommand{\xref}{\tilde{\vxx}}
\newcommand{\dMR}{\dot{\MR}}
\newcommand{\dV}{\dot{V}}
\newcommand{\dvs}{\dot{\vs}}

\newcommand{\bmat}{\left[ \begin{array}}
\newcommand{\emat}{\end{array}\right]}
\newcommand{\MRgt}{\MR^\circ}
\newcommand{\vtgt}{\vt^\circ}

%% file: abstract.tex

\begin{abstract}
  We provide a dynamical perspective on the classical problem of 3D \emph{point cloud registration with correspondences}. A point cloud is considered as a \emph{rigid body consisting of particles}. The problem of registering two point clouds is formulated as a dynamical system, where the \emph{dynamic} model point cloud \emph{translates and rotates} in a \emph{viscous} environment towards the \emph{static} scene point cloud, under forces and torques induced by \emph{virtual springs} placed between each pair of corresponding points. We first show that the \emph{potential energy} of the system recovers the \emph{objective function} of the \emph{maximum likelihood estimation}. We then adopt Lyapunov analysis, particularly the \emph{invariant set theorem}, to analyze the rigid body dynamics and show that the system globally asymptotically tends towards the \emph{set of equilibrium points}, where the \emph{globally optimal} registration solution lies in. We conjecture that, besides the globally optimal equilibrium point, the system has either \emph{three} or \emph{infinite} ``spurious'' equilibrium points, and these spurious equilibria are all \emph{locally unstable}. The case of three spurious equilibria corresponds to \emph{generic} shape of the point cloud, while the case of infinite spurious equilibria happens when the point cloud exhibits \emph{symmetry}. Therefore, simulating the dynamics with random perturbations guarantees to obtain the globally optimal registration solution. Numerical experiments support our analysis and conjecture. 
\end{abstract}

%% file: intro.tex

\section{Introduction}
\label{sec:intro}
\emph{Point cloud registration}, which seeks the best 3D rigid transformation,~\ie~rotation and translation, to align two point clouds, is a fundamental problem in robotics and computer vision and finds applications in motion estimation and 
3D reconstruction~\cite{Henry12ijrr-rgbdMapping,Blais95pami-registration,Choi15cvpr-robustReconstruction,Zhang15icra-vloam,Yang20arxiv-teaser},
object recognition and localization~\cite{Drost10cvpr,Wong17iros-segicp,Zeng17icra-amazonChallenge,Marion18icra-labelFusion,Yang19rss-teaser}, 
panorama stitching~\cite{Bazin14eccv-robustRelRot,Yang19iccv-QUASAR},  
and medical imaging~\cite{Audette00mia-surveyMedical,Tam13tvcg-registrationSurvey}.

We revisit the simplest setup where point-to-point correspondences are \emph{known and correct}. Formally, let $\calX = \{ \vxx_i \}_{i=1}^N$ and $\calY = \{ \vy_i \}_{i=1}^N,\vxx_i,\vy_i \in \Real{3}$, be two point clouds with $N$ correspondences $\vxx_i \leftrightarrow \vy_i$ (\eg~obtained from hand-crafted~\cite{Rusu09icra-fast3Dkeypoints} or deep-learned feature matching~\cite{choy19iccv-FCGF}), then it is well known that, if the noise follows isotropic Gaussian distribution, the \emph{maximum likelihood estimator} (MLE) for the rotation $\MR \in \SOthree$\footnote{$\SOthree \doteq \{ \MR \in \Real{3\times 3}: \MR\tran\MR = \MR\MR\tran = \eye_3, \det(\MR) = +1 \}$, where $\eye_d$ is the identity matrix of size $d$.} and translation $\vt \in \Real{3}$ is given by:
\bea
\MR^\star, \vt^\star = \argmin_{\MR \in \SOthree, \vt \in \Real{3}} \sum_{i=1}^N \frac{1}{\sigma_i^2} \left\| \vy_i - \MR \vxx_i - \vt  \right\|^2, \label{eq:registrationOutlier-free}
\eea
where $\sigma_i$ is the standard deviation of the Gaussian noise associated with the $i$-th correspondence. Despite the nonconvexity of problem~\eqref{eq:registrationOutlier-free} ($\SOthree$ is a nonconvex set), its \emph{globally optimal} solution can be computed in closed form due to Horn~\cite{Horn87josa} and Arun~\etal~\cite{Arun87pami}, by first computing the optimal rotation using singular value decomposition (SVD) and then computing the translation analytically. Although problem~\eqref{eq:registrationOutlier-free} requires all correspondences to be correct, it is the most fundamental building block for robust estimation frameworks such as RANSAC~\cite{Fischler81}, M-estimation~\cite{Zhou16eccv-fastGlobalRegistration} and graduated non-convexity~\cite{Yang20ral-GNC} in the presence of outlier correspondences. In addition, when no correspondences are given, the \emph{Iterative Closest Point} (ICP) algorithm alternates in finding correspondences and updating the transformation from the solution of problem~\eqref{eq:registrationOutlier-free}. 

Recently, physics-based registration~\cite{Golyanik16CVPR-gravitationalRegistration,Jauer18PAMI-physicsBasedRegistration,Golyanik19ICCV-acceleratedGravitational} has become a popular alternative to optimization-based registration. By creating \emph{virtual forces} such as gravitational force~\cite{Golyanik16CVPR-gravitationalRegistration,Golyanik19ICCV-acceleratedGravitational} and electrostatic force~\cite{Jauer18PAMI-physicsBasedRegistration} between two point clouds, one of which is static and the other is dynamically moving (rotating and translating) under the virtual forces as a \emph{rigid body of particles}, point cloud registration can be solved from \emph{simulating the differential equations derived from rigid body dynamics}. The appealing advantage of solving registration using simulation lies in its computational efficiency: the simulation of dynamics can be efficiently performed using GPUs with rich literature on implementation and acceleration~\cite{Golyanik19ICCV-acceleratedGravitational}. Nevertheless, little has been analyzed about the dynamics of the moving point cloud under virtual forces, and the convergence properties of simulation-based methods are unknown. In fact, simulation-based techniques are also prone to getting stuck at local minima and hence good initialization is required~\cite{Jauer18PAMI-physicsBasedRegistration}.

{\bf Contributions.} In this paper, we focus on providing a theoretical understanding of the dynamical system arising from point cloud registration and its connection to optimization-based registration. Our first contribution is, instead of creating virtual gravitational and electrostatic forces as in previous works~\cite{Golyanik16CVPR-gravitationalRegistration,Jauer18PAMI-physicsBasedRegistration}, we place \emph{virtual springs} between each pair of corresponding points with the spring coefficient proportional to the inverse of the square of the standard deviation of the noise (\ie~$1/\sigma_i^2$ in~\eqref{eq:registrationOutlier-free}). Under this construction, the \emph{potential energy} of the dynamical system exactly recovers the objective function of problem~\eqref{eq:registrationOutlier-free}. Our second contribution is to leverage Lyaponov theory, in particular the \emph{invariant set theorem}~\cite{Slotine91book-appliedNonlinearControl}, to analyze the stability of this system. We show that the system, starting from any initial conditions, will tend to the set of \emph{equilibrium points}, which must contain the globally optimal solution of problem~\eqref{eq:registrationOutlier-free}. This result is not enough to guarantee that \emph{simulating the dynamics} will converge to the globally optimal solution. Therefore, our third contribution is to conjecture and numerically show that, besides the equilibrium point corresponding to the optimal solution of problem~\eqref{eq:registrationOutlier-free}, the system contains either \emph{three} or \emph{infinite} ``spurious'' equilibrium points, and all spurious equilibria are \emph{locally unstable}. When the point cloud has \emph{generic} shape and noise distribution, the dynamical system has three spurious equilibria, while the case of infinite spurious equilibria happens when the point cloud exhibits \emph{symmetry}. The conjecture suggests that a small perturbation is sufficient for the system to escape these spurious ``suboptimal'' equilibria. In fact, our simulation of the dynamics demonstrates that, even \emph{without} perturbation, the system \emph{always} converges to the globally optimal solution of problem~\eqref{eq:registrationOutlier-free}, supporting our conjecture. Moreover, numerical experiments suggest that the rotations corresponding to the three locally unstable equilibrium points are related to the globally optimal rotation via an additional $180^\circ$ rotation. Although our analysis only focuses on point cloud registration in its simplest form~\eqref{eq:registrationOutlier-free}, this is the first theoretical understanding of physics-based registration and we hope future work can be built upon this. 

The rest of the paper is organized as follows. Section~\ref{sec:dynamicalConstruction} describes the construction of a dynamical system for point cloud registration and the expression for the rigid body dynamics. Section~\ref{sec:stabilityAnalysis} presents the stability analysis of the dynamical system using Lyapunov theory. Section~\ref{sec:experiments} provides numerical experiments supporting our analysis. Section~\ref{sec:conclusions} concludes the paper. Due to the focus of this paper, we omit a section dedicated to related work, and point the interested readers to~\cite{Yang20arxiv-teaser,Jauer18PAMI-physicsBasedRegistration} for thorough reviews of related work on point cloud registration.

%% file: dynamicalConstruction.tex
\section{Dynamical Construction for Point Cloud Registration}
\label{sec:dynamicalConstruction}
We consider each point cloud $\calX$ and $\calY$ as a collection of $N$ particles each with mass $m_i=1/\sigma_i^2$ (red and blue solid circles in Fig.~\ref{fig:alignment-dynamics}), and the particles are mutually rigidly connected via massless links (red and blue lines in Fig.~\ref{fig:alignment-dynamics}). We are given the coordinates of the particles $\vxx_i,\vy_i$ at $t=0$ (with slight abuse of notation, we also use $\vxx_i$ and $\vy_i$ to refer to the particles). The center of mass (CM) of $\calY$ is located at $\bar{\vy} = \sum_{i=1}^N m_i \vy_i / M$, and the CM of $\calX$ is located at $\xcom = \sum_{i=1}^N m_i\vxx_i / M$, where $M \doteq \sum_{i=1}^N m_i$ is the total mass of $\calY$ (also the total mass of $\calX$). Without loss of generality, we assume $\bar{\vy} = \zero$ and $\calY$ is already centered at zero.\footnote{Otherwise, we can shift every point by $\bar{\vy}$: $\vy_i \leftarrow \vy_i - \bar{\vy},\vxx_i \leftarrow \vxx_i - \bar{\vy}$.} As shown in Fig.~\ref{fig:alignment-dynamics}, we then attach a \emph{fixed} global coordinate frame $Y$ to the CM of $\calY$, and attach a moving coordinate frame $X$ to the CM of $\calX$. Point cloud $\calY$ is static for all $t\geq 0$ and point cloud $\calX$ is allowed to move. 

{\bf Virtual springs and damping.} We place a virtual spring with coefficient $k_i \doteq 2 / \sigma_i^2$ between each pair of corresponding particles $\vxx_i$ and $\vy_i$, and we assume that the particles are subject to \emph{viscous damping} with constant coefficient $\mu > 0$. Under spring forces and viscous damping, $\calX$ will undergo rigid body dynamics and moves in the virtual environment.

{\bf State space of the dynamical system.} We use the following states to describe the dynamical system (blue symbols in Fig.~\ref{fig:alignment-dynamics}): (i) $\xcom^Y(t) \in \Real{3}$, the position of the CM of $\calX$ in the coordinate frame $Y$; (ii) $\MR(t) \doteq \MR_X^Y (t) \in \SOthree$, the relative orientation between frame $X$ and frame $Y$; (iii) $\xcomv^Y(t) \in \Real{3}$, the velocity of the CM of $\calX$ in frame $Y$; (iv) $\xangv^X(t) \in \Real{3}$, the angular velocity of $\calX$ \wrt its CM in its body frame $X$. In order to describe the position and velocity of each particle $\vxx_i$, we use $\xref_i \doteq \vxx_i^X - \xcom^X = \vxx_i - \xcom$ to denote the relative position of each particle \wrt the CM in frame $X$, which is constant for all $t \geq 0$ by definition. Then the position of each particle $\vxx_i$ in the frame $Y$ can be expressed as:
\bea
\vxx_i^Y(t) = \MR(t) \xref_i + \xcom^Y(t). \label{eq:positioninY}
\eea
With this construction, we consider the following initial condition for the dynamical system:
\bea
\vxx_i^Y(0) = \vxx_i, \quad \xcom^Y(0) = \xcom = \frac{\sumallpoints m_i \vxx_i}{M}, \quad \xcomv^Y(0) = \zero, \quad \MR(0) = \eye, \quad \xangv^X(0) = \zero,
\eea
which means that point cloud $\calX$ is at rest position and frame $X$ is axis-aligned with frame $Y$ (with an offset in the location of the CM).

{\bf Forces and torques.} We now derive the total force and torque applied on $\calX$ induced by the virtual springs and damping. The force acting on each point $\vxx_i$, expressed in frame $Y$, is:
\bea
\vf_i^Y = k_i (\vy_i - \vxx_i^Y) - \mu m_i \vv_i^Y = k_i (\vy_i - \MR \xref_i - \xcom^Y) - \mu m_i (\MR(\xangv^X \times \xref_i) + \xcomv^Y), \label{eq:forcepointiY}
\eea
where we have assumed the damping force is also proportional to the mass $m_i$ of the particle.\footnote{Under the assumption that the particles have equal density, then larger mass implies larger volume, and hence larger damping force.} Hence the total force is:
\bea
\vf^Y = \sumallpoints \vf_i^Y = \sumallpoints k_i \left( \vy_i - \MR \xref_i - \xcom^Y  \right) - \mu M \xcomv^Y, \label{eq:totalforce}
\eea
where we have used:
\bea
\sumallpoints \mu m_i \MR (\xangv^X \times \xref_i ) = \sumallpoints \mu \MR \xangv^X \times (m_i \xref_i) = \mu\MR \xangv^X \times \left( \sumallpoints m_i \xref_i \right) = \zero. \label{eq:vanishCM1}
\eea
due to the definition of $\xref_i = \vxx_i - \xcom$. 
From the expression of $\vf_i^Y$ in eq.~\eqref{eq:forcepointiY}, we can find the expression of $\vf_i^X$:
\bea
\vf_i^X = \MR\tran \vf_i^Y = k_i \left( \MR\tran \vy_i - \xref_i - \MR\tran \xcom^Y \right) - \mu m_i \left( \xangv^X \times \xref_i + \MR\tran \xcomv^Y \right).
\eea
Hence, we can compute the total torque:
\bea
\vtau^X = \sumallpoints \xref_i \times \vf_i^X = \sumallpoints \left( k_i \hatmap{\xref_i} \MR\tran \left( \vy_i - \xcom^Y \right) - \mu m_i \hatmap{\xref_i}\hatmap{\xangv^X}\xref_i \right), \label{eq:totaltorque}
\eea
where we have used:
\bea
\sumallpoints \mu m_i \xref_i \times (\MR\tran \xcomv^Y) = \mu \left( \sumallpoints m_i \xref_i \right) \times (\MR\tran \xcomv^Y) = \zero,
\eea
for the same reason as~\eqref{eq:vanishCM1}. In the expression~\eqref{eq:totaltorque}, the linear map $\hatmap{\vxx}$ maps $\vxx \in \Real{3}$ to the following skew-symmetric matrix:
\bea
\hatmap{\vxx} = \bmat{ccc}
0 & -x_3 & x_2 \\
x_3 & 0 & - x_1 \\
- x_2 & x_1 & 0
\emat,
\eea
and we have the equality that $\vxx \times \vy = \hatmap{\vxx} \vy$ for any $\vy \in \Real{3}$. 

{\bf Dynamics of the system.} Now we are ready to introduce the dynamics of the system.
\begin{proposition}[Dynamics of Point Cloud Registration]\label{prop:dynamics}
Using $\left(\xcom^Y,\MR,\xcomv^Y,\xangv^X \right)$ as the state space of $\calX$, the differential equations that describe the dynamics of $\calX$ are:
\bea \label{eq:dynamics}
\begin{cases}
\dxcom^Y = \xcomv^Y\\
\dMR = \MR \hatmap{\xangv^X} \\
M \xcoma^Y = \vf^Y \\
\MJ \xanga^X = \vtau^X - \xangv^X \times \MJ \xangv^X
\end{cases},
\eea
where $\xcoma^Y \doteq \dot{\xcomv}^Y$ is the linear acceleration of $\calX$ in frame $Y$, $\xanga^X \doteq \dot{\xangv}^X$ is the angular acceleration of $\calX$ in frame $X$, $\MJ = - \sumallpoints m_i \hatmap{\xref_i}^2 \in \Real{3 \times 3}$ is the inertia matrix of $\calX$ \wrt its CM, and $\vf^Y,\vtau^X$ are the total force and torque in~\eqref{eq:totalforce} and~\eqref{eq:totaltorque}. Further, let $\vr = \vectorize{\MR} \in \Real{9}$, and use $\vs = [\xcom^Y; \vr; \xcomv^Y; \xangv^X ] \in \Real{18}$ to denote the vector of states, we also write $\dvs = \calF(\vs)$ to denote the nonlinear dynamics in~\eqref{eq:dynamics}.
\end{proposition}
Proposition~\ref{prop:dynamics} is straightforward from Newton-Euler equations for rigid body dynamics~\cite{Asada86book-robotAnalysis}. We note that the inertia matrix $\MJ$ is symmetric positive definite and invertible when there are at least three noncollinear points in $\calX$.

\input{fig-alignment_dynamics}

%% file: fig-alignment_dynamics.tex

\begin{figure}[t]
\centering
\includegraphics[width=0.5\linewidth]{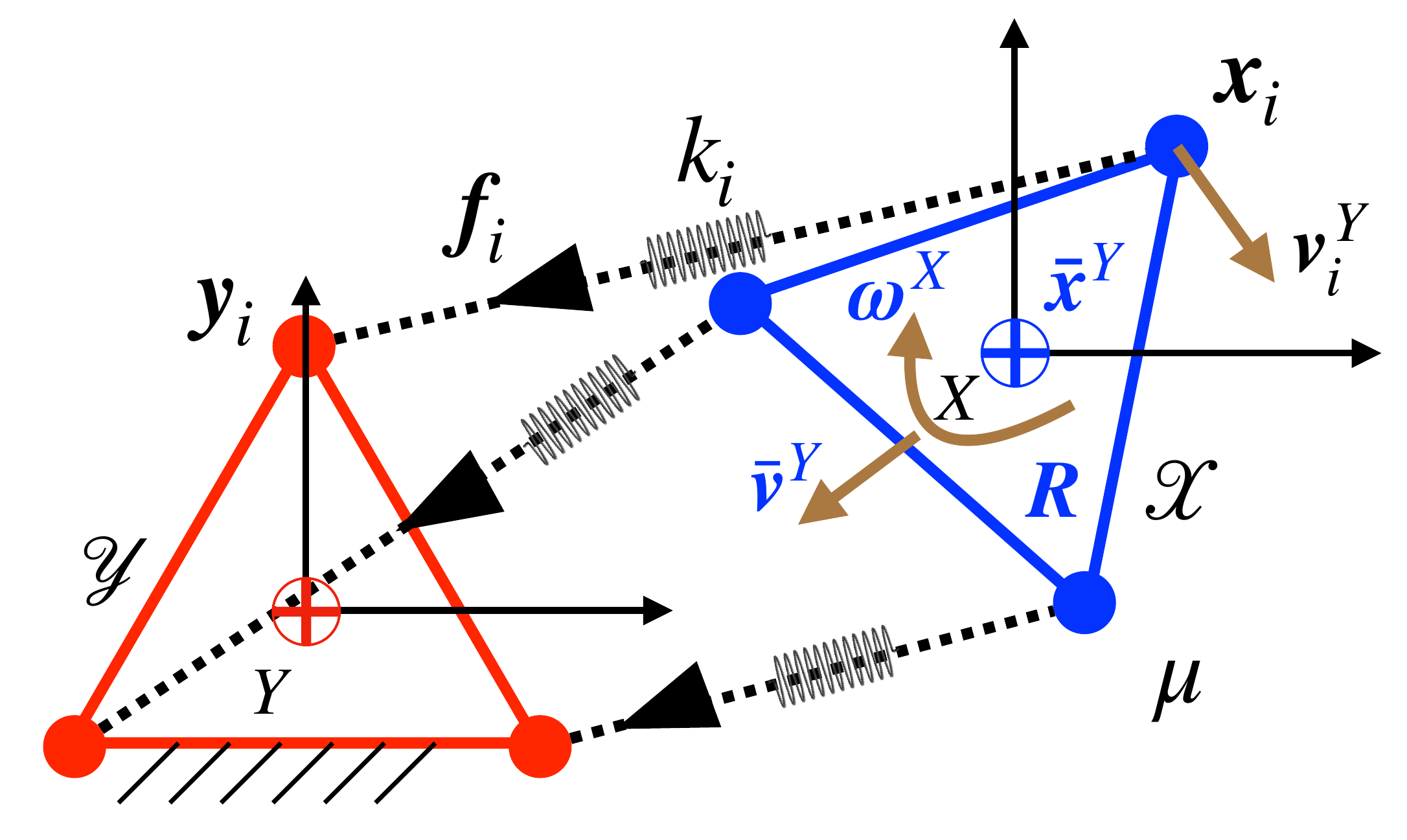}
\caption{Point cloud registration as a dynamical system with virtual spring forces. By placing a virtual spring with coefficient $k_i = 2 / \sigma_i^2$ between each pair of points $(\vxx_i,\vy_i)$, the static scene $\calY$ (\red{red}) attracts the moving model $\calX$ (\blue{blue}) to optimal alignment via virtual forces and torques. The energy of the system is dissipated through viscous damping with constant coefficient $\mu$.}
\label{fig:alignment-dynamics}
\end{figure}

%% file: stabilityAnalysis.tex
\section{Stability Analysis}
\label{sec:stabilityAnalysis}
In this section, we apply Lyapunov theory, in particular the \emph{invariant set theorem}, to analyze the stability of the nonlinear dynamics $\dvs = \calF(\vs)$ in~\eqref{eq:dynamics}. We first introduce some preliminaries. 

\begin{definition}[Invariant Set~\cite{Slotine91book-appliedNonlinearControl}]\label{def:invariantSet} A set $\calI$ is an invariant set for a dynamic system if every system trajectory which starts from a point in $\calI$ remains in $\calI$ for all future time.
\end{definition}

A first trivial example for an invariant set is the entire state space, and a second trivial example for an invariant set is an \emph{equilibrium point},~\ie~a point at which $\dvs = \zero$. The following theorem establishes when the dynamical system will converge to an invariant set. 

\begin{theorem}[Global Invariant Set Theorem~\cite{Slotine91book-appliedNonlinearControl}]\label{thm:globalinvariant} 
Consider the dynamical system $\dvs = \calF(\vs)$, with $\calF$ continuous, and let $V(\vs)$ be a scalar function with continuous first partial derivatives. Assume that (i) $V(\vs) \rightarrow \infty$ as $\| \vs \| \rightarrow \infty$; (ii) $\dV(\vs) \leq 0$ over the entire state space. Let $\calR$ be the set of all points where $\dV(\vs) = 0$, and $\calI$ be the largest invariant set in $\calR$. Then all trajectories of the system globally asymptotically converge to $\calI$ as $t \rightarrow \infty$.
\end{theorem}

The crux of Theorem~\ref{thm:globalinvariant} is to find a scalar function $V(\vs)$, often referred to as the \emph{Lyapunov function} or the \emph{energy function}, and the set $\dV(\vs) = 0$ defines an interesting set. Interestingly, by choosing $V(\vs)$ to be the total energy of the system, we obtain the following result.

\begin{proposition}[Convergence to Equilibria]\label{prop:convergenceEquilibria}
Let $V(\vs) = V_k(\vs) + V_p(\vs)$ be the total energy of the dynamical system $\dvs = \calF(\vs)$ described in~\eqref{eq:dynamics}, where $V_k(\vs)$ is the total kinetic energy and $V_p(\vs)$ is the total potential energy. Then (i) $V_p(\vs)$ exactly recovers the objective function of problem~\eqref{eq:registrationOutlier-free}; (ii) all trajectories of the system globally asymptotically converge to the set of equilibrium points $\calE = \{\vs: \dvs = \calF(\vs) = \zero\}$, where the total force and torque both vanish, and the origin of $X$ aligns with the origin of $Y$,~\ie~$\xcom^Y = \bar{\vy} = \zero$.
\end{proposition}
\begin{proof}
The total kinetic energy of the system is:
\bea
V_k(\vs) = \sumallpoints \frac{m_i}{2} \left\| \vv_i^Y \right\|^2 = \sumallpoints \frac{m_i}{2} \left\| \xcomv^Y + \MR\left( \xangv^X \times  \xref_i \right) \right\|^2 = \frac{M}{2}\left\| \xcomv^Y \right\|^2 + \half \left( \xangv^X \right)\tran \MJ \xangv^X, \label{eq:kineticenergy}
\eea
where the first term is the translational energy and the second term is the rotational energy \wrt its CM. The total potential energy of the system comes purely from spring forces:
\bea
V_p(\vs) = \sumallpoints \frac{k_i}{2} \left\| \vy_i - \MR \xref_i - \xcom^Y \right\|^2, \label{eq:potentialenergy}
\eea
which exactly recovers the objective function of problem~\eqref{eq:registrationOutlier-free} because $k_i = 2/\sigma_i^2$.\footnote{Note that since $\xref_i = \vxx_i - \xcom$, we have $\vy_i - \MR \xref_i - \xcom^Y = \vy_i - \MR\vxx_i - \left( \xcom^Y - \MR \xcom \right)$. Therefore, we can simply let $\vt = \xcom^Y - \MR \xcom$ to recover problem~\eqref{eq:registrationOutlier-free}. } The Lyapunov function is the total energy:
\bea
V(\vs) = V_k(\vs) + V_p(\vs) = \frac{M}{2}\left\| \xcomv^Y \right\|^2 + \half \left( \xangv^X \right)\tran \MJ \xangv^X + \sumallpoints \frac{k_i}{2} \left\| \vy_i - \MR \xref_i - \xcom^Y \right\|^2,
\eea 
which apparently satisfies $V(\vs) \rightarrow \infty$ as $\| \vs \| \rightarrow \infty$.\footnote{Because there are at least three noncollinear measurements $\vxx_i$'s, there are at least three nonparallel $\xref_i$'s. Hence, as $\MR \rightarrow \infty$, at least one of the $\| \MR\xref_i \|$ will go to infinity.} We then compute the derivative of $V(\vs)$, using the dynamic equations~\eqref{eq:dynamics}. The derivative of the kinetic energy is:
\bea
& \dV_k(\vs) = \left( \xcomv^Y \right)\tran \left( M \xcoma^Y \right) + \left( \xangv^X \right)\tran \left(\MJ \xanga^X \right) = \left( \xcomv^Y \right)\tran \vf^Y + \left( \xangv^X \right)\tran \left( \vtau^X - \xangv^X \times \MJ \xangv^X \right) \\
& \scriptstyle = \left( \xcomv^Y \right)\tran \left( \sum_i k_i \left( \vy_i - \MR \xref_i - \xcom^Y  \right) - \mu M \xcomv^Y \right) + \left( \xangv^X \right)\tran \left( \sum_i \left( k_i \hatmap{\xref_i} \MR\tran \left( \vy_i - \xcom^Y \right) - \mu m_i \hatmap{\xref_i}\hatmap{\xangv^X}\xref_i \right) \right) \\
& \scriptstyle = - \mu M \| \xcomv^Y \|^2 - \mu \left( \xangv^X \right) \tran \MJ \left( \xangv^X \right) + \sum_i k_i \left( \xcomv^Y \right)\tran \left( \vy_i - \MR \xref_i - \xcom^Y  \right) + 
 \sum_i k_i \left(\MR \hatmap{\xangv^X} \xref_i \right)\tran \left( \vy_i - \xcom^Y \right),
\eea
and the derivative of the potential energy is:
\bea
\dV_p(\vs) = \sum_i k_i \left( \vy_i - \MR \xref_i - \xcom^Y \right)\tran\left( - \MR\hatmap{\xangv^X} \xref_i - \xcomv^Y \right),
\eea
and therefore the derivative of the total energy is:
\bea
\dV(\vs) = - \mu M \| \xcomv^Y \|^2 - \mu \left( \xangv^X \right) \tran \MJ \left( \xangv^X \right) \leq 0.\label{eq:Vdot} 
\eea
We remark that one can obtain $\dV(\vs)$ in~\eqref{eq:Vdot} directly from physics: $\dV(\vs)$ is the \emph{energy dissipation rate}, and the only energy dissipation of this system comes from viscous damping. From the expression of $\dV(\vs)$, we know the set $\calR = \{ \vs: \dV(\vs) = 0\}$ is precisely the set $\calR = \{\vs: \xcomv^Y = \zero, \xangv^X = \zero\}$ because $\dV(\vs) = 0$ if and only if both velocities $\xcomv^Y$ and $\xangv^X$ vanish. Now consider the invariant set within $\calR$: because both $\xcomv^Y=\zero$ and $\xangv^X = \zero$, for any point to stay inside $\calR$, the translational and rotational accelerations must be zero,~\ie~$\xcoma^Y = \zero$, $\xanga^X = \zero$. In other words, the largest invariant set in $\calR$ is precisely the set of \emph{equilibrium points} $\calE = \{ \vs: \dvs = \calF(\vs) = \zero\}$, and from the global invariant set theorem~\ref{thm:globalinvariant}, we know that the system tends towards $\calE$ from any initial state. Eventually, from the dynamics~\eqref{eq:dynamics}, we know that any point in $\calE$ satisfies:
\bea
\calE:
\begin{cases}
\sumallpoints k_i \left( \vy_i - \MR \xref_i - \xcom^Y \right) = \zero \\
\sumallpoints k_i \hatmap{\xref_i} \MR\tran \left( \vy_i - \xcom^Y \right) = \zero \\
\xcomv^Y = \zero \\
\xangv^X = \zero
\end{cases}
\Longrightarrow 
\calE:
\begin{cases}
\xcom^Y = \bar{\vy} = \zero \\
\sumallpoints k_i \hatmap{\xref_i} \MR\tran \vy_i = \zero \\
\xcomv^Y = \zero \\
\xangv^X = \zero
\end{cases}, \label{eq:invariantSet}
\eea
where the total force and torque applied on $\calX$ are both zero. From the second equation in~\eqref{eq:invariantSet}, we know that the origin of $X$ aligns with the origin of $Y$.
\end{proof}

An immediate corollary states that the unique globally optimal minimizer of the potential energy $V_p$ written in~\eqref{eq:potentialenergy} (and hence, the globally optimal solution to problem~\eqref{eq:registrationOutlier-free}) lies inside the set of equilibrium points.

\begin{corollary}[Global Optimizer]\label{coro:globalOptimizer}
If the measurements contain at least three noncollinear points $\vxx_i$'s, then the potential energy $V_p$ in~\eqref{eq:potentialenergy} admits a unique global minimizer $(\xcom^{Y\star},\MR^\star)$. Let $\vs^\star = [\xcom^{Y\star};\vr^\star;\zero_{3\times1};\zero_{3\times1}]$, then $\vs^\star$ is contained in the set of equilibrium points $\calE = \{\vs: \dvs = \calF(\vs) = \zero \}$.
\end{corollary}
\begin{proof}
From~\cite{Horn87josa}, we know that the potential energy admits a unique closed-form global optimizer when there are at least three noncollinear measurements. Since $\xcomv^Y=\zero, \xangv^X = \zero$ is the unique minimizer of the kinetic energy $V_k$ in~\eqref{eq:kineticenergy}, it follows that $\vs^\star$ is the unique global minimizer of the total energy $V$. Clearly $\dV(\vs^\star) = \zero$ and $\vs^\star \in \calR$. Therefore $\vs^\star \in \calE$ must hold. Otherwise, suppose $\vs^\star \not\in \calE$, and $\vs^\star$ becomes $\hat{\vs}$ at the next time step. From $\dV \leq 0$, we have $V(\hat{\vs}) \leq V(\vs^\star)$, contradicting the fact that $\vs^\star$ is the unique global minimizer of $V$.  
\end{proof}

Proposition~\ref{prop:convergenceEquilibria} and Corollary~\ref{coro:globalOptimizer} together imply that the dynamical system~\eqref{eq:dynamics} will globally asymptotically tend to the set of equilibrium points, which contains the unique global minimizer of the potential energy as a specific solution. However, this result is insufficient to conclude that simulating the dynamics~\eqref{eq:dynamics} can solve the optimization problem~\eqref{eq:registrationOutlier-free}, because there could exist many equilibrium points. The next conjecture states that the system does have spurious equilibrium points other than the global minimizer in Corollary~\ref{coro:globalOptimizer}, but all spurious equilibrium points are locally unstable. 

\begin{conjecture}[Characterization of Equilibria]\label{conj:characterizeEquilibria} Besides the equilibrium point corresponding to the global minimizer of the energy function as stated in Corollary~\ref{coro:globalOptimizer}, the dynamical system $\dvs = \calF(\vs)$ in~\eqref{eq:dynamics} has either three or infinite spurious equilibrium points. Moreover, 
\begin{enumerate}[label=(\roman*)]
	\item \label{conj:three} when the point cloud has generic shape,~\ie~all points $\vxx_i$ are randomly generated, the system has three spurious equilibria, and the rotations at all three equilibria differ from the globally optimal rotation by $180^\circ$;
	\item \label{conj:infinite} when the point cloud exhibits symmetry, the system has infinite spurious equilibria;
	\item all spurious equilibria are \emph{locally unstable}.
\end{enumerate} 
\end{conjecture} 
\begin{proof}
We provide a graphical proof for~\ref{conj:infinite} in the case of $N=3$ and $N=4$, as shown in Fig.~\ref{fig:symmetry}. When $N=3$ (Fig.~\ref{fig:symmetry}(a)), consider both $\calX$ (blue) and $\calY$ (red) are equilateral triangles with $l$ being the length from the vertex to the center. Assume the particles have equal masses such that the CM is also the geometric center $O$, and all virtual springs have equal coefficients. $\calX$ is obtained from $\calY$ by first rotating counter-clockwise (CCW) around $O$ with angle $\theta$, and then flipped about the line that goes through point 1 and the middle point between point 2 and 3. We will show that this is an equilibrium point of the dynamical system for any $\theta$. When the CM of $\calX$ and the CM of $\calY$ aligns, we know the forces $\vf_i,i=1,2,3$ are already balanced. It remains to show that the torques $\vtau_i,i=1,2,3$ are also balanced for any $\theta$. $\vtau_1$ and $\vtau_3$ applies clockwise (CW, cyan) and the value of their sum is:
\bea
\hspace{-10mm} \| \vtau_1 + \vtau_3 \| = \| \vtau_1 \| + \| \vtau_3 \| = kl^2 \left( \sin \theta +  \sin\beta \right) = kl^2 \left( \sin\theta + \sin\left(\theta + \frac{2\pi}{3} \right) \right) = kl^2 \sin\left( \theta + \frac{\pi}{3} \right),
\eea
and $\tau_2$ applies CCW (green) and its value is:
\bea
\| \vtau_2 \| = kl^2 \sin\alpha  = kl^2 \sin\left( \frac{2\pi}{3} - \theta \right) = kl^2 \sin\left( \theta + \frac{\pi}{3} \right).
\eea
Therefore, the torques cancel with each other and the configuration in Fig.~\ref{fig:symmetry}(a) is an equilibrium state for all $\theta$. However, it is easy to observe that this type of equilibrium is unstable because any perturbation that drives point 2 out of the 2D plane will immediately drives the system out of this type of equilibrium. When $N=4$, one can verify that same torque cancellation happens: 
\bea
\| \vtau_1 \| = kl^2 \sin\beta = kl^2 \sin\left( \theta + \frac{\pi}{2} \right) =  kl^2 \sin\left( \frac{\pi}{2} - \theta \right) = kl^2 \sin\alpha = \| \vtau_3 \|,
\eea
and the system also has infinite locally unstable equilibria.
\end{proof}

\input{fig-symmetry}

The (partial) proof above justifies~\ref{conj:infinite}. To verify~\ref{conj:three} of Conjecture~\ref{conj:characterizeEquilibria}, we numerically solve the equilibrium equations~\eqref{eq:invariantSet}, and compute the eigenvalues of the Jacobian $\MA = \partial \calF /\partial \vs$ at each of the solutions. From Lyapunov's linearization method~\cite{Slotine91book-appliedNonlinearControl}, if the Jacobian $\MA$ has at least one eigenvalue \emph{strictly in the right-half complex plane}, then it serves as a certificate for local instability of the equilibrium point. In Section~\ref{sec:experiments}, we show that solving the system of polynomial equations~\eqref{eq:invariantSet} always yields four solutions, and the Jacobian $\MA$ at three of the spurious solutions always has eigenvalues with positive real parts. 

If Conjecture~\ref{conj:characterizeEquilibria} were true, then simulating the dynamics has high probability of converging to the globally optimal solution. In addition, if the simulation does converge to a locally unstable equilibrium, then computing the Jacobian can inform this unexpected convergence and adding a small perturbation can help the simulation escape this locally unstable equilibrium until it reaches the true global optimizer. In fact, in Section~\ref{sec:experiments}, all the experiments converge to the global optimizer.

%% file: fig-symmetry.tex

\newcommand{\mpwthree}{6cm}
\newcommand{\myhspace}{\hspace{-3mm}}

\begin{figure}[t]
	\begin{center}
	\begin{minipage}{\textwidth}
	\begin{tabular}{cc}%
			\begin{minipage}{\mpwthree}%
			\centering%
			\includegraphics[width=0.9\columnwidth]{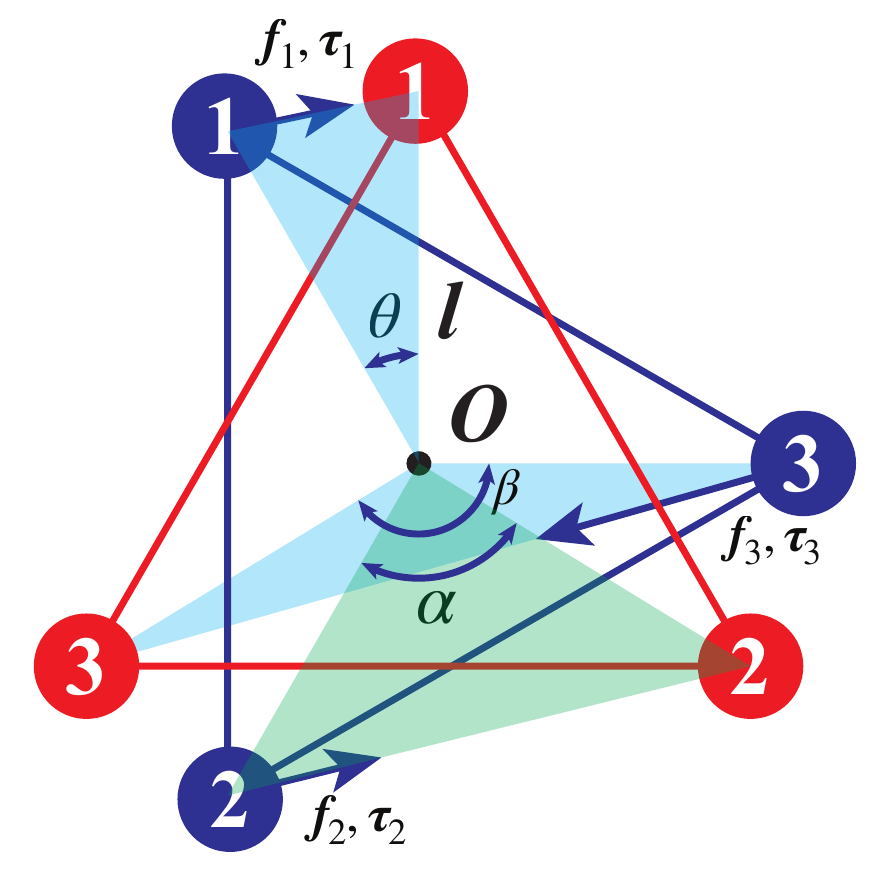} \\
			{\small (a) Equilateral triangle.}
			\end{minipage}
		& \hspace{6mm}
			\begin{minipage}{\mpwthree}%
			\centering%
			\includegraphics[width=0.9\columnwidth]{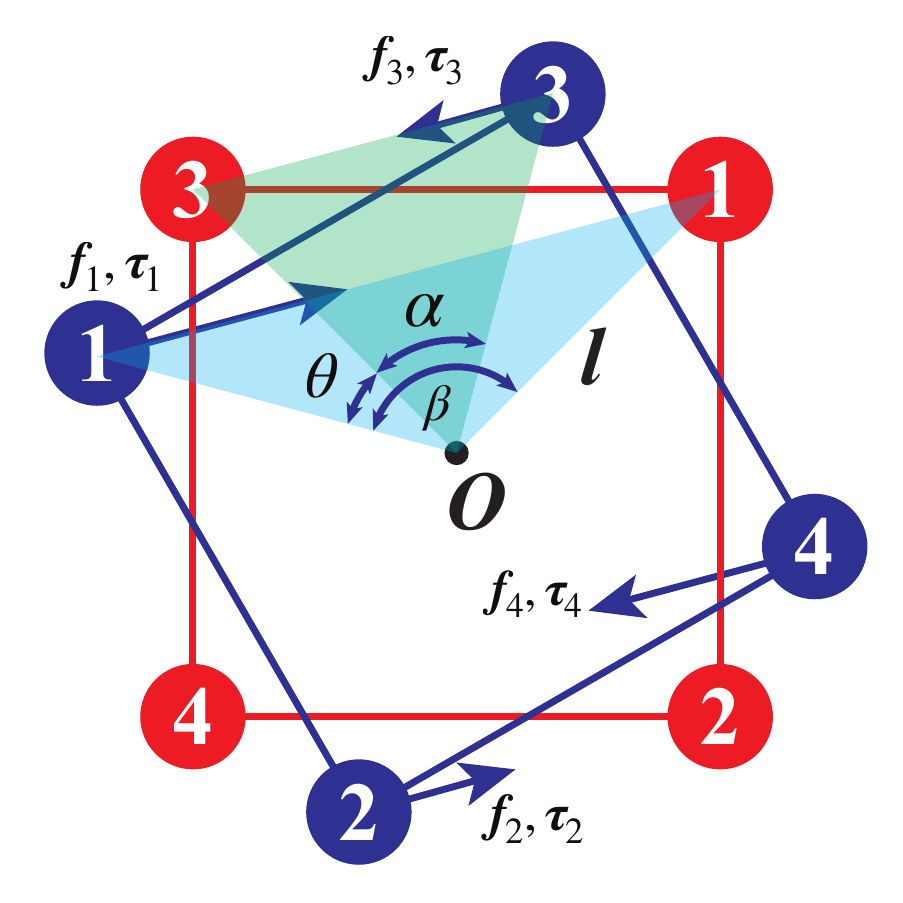} \\
			{\small (b) Square.}
			\end{minipage}
	\end{tabular}
	\end{minipage} 
	\caption{Examples of symmetric point clouds: (a) an equilateral triangle and (b) a square. The dynamical system has infinite equilibrium points.
	\label{fig:symmetry}} 
	\vspace{-5mm} 
	\end{center}
\end{figure}

%% file: experiments.tex
\section{Numerical Experiments}
\label{sec:experiments}

\subsection{Characterization of Equilibria}
\label{sec:equilibria}
We first run numerical experiments to justify Conjecture~\ref{conj:characterizeEquilibria}. At each Monte Carlo run, we first randomly generate a point cloud $\calX$, where each point follows an isotropic Gaussian distribution with standard deviation 1,~\ie~$\vxx_i \sim \calN(\zero,\eye_3)$. Then we randomly samples a rotation matrix $\MR$ and a translation vector $\vt \sim \calN(\zero,\eye_3)$, and obtain the point cloud $\calY$ by: $\vy_i = \MR \vxx_i + \vt + \vepsilon_i$, where $\vepsilon_i \sim \calN(\zero,\sigma_i^2 \eye_3)$ models the Gaussian noise with $\sigma_i = 0.01$. Then we shift both $\calX$ and $\calY$ by $\bar{\vy}$ so that $\calY$ is centered at the origin, so that the ground-truth transformation is now: $\MRgt = \MR, \vtgt = \vt + \MR \bar{\vy} - \bar{\vy}$. We then use Matlab to solve the second equation in~\eqref{eq:invariantSet}: $\sumallpoints k_i \hatmap{\xref_i} \MR\tran \vy_i = \zero$, together with the constraint that $\MR \in \SOthree$, which boils down to a set of quadratic polynomial equalities in the entries of $\MR$ (see~\eg~\cite{Yang20cvpr-shapeStar} for the expressions). This gives us the set of equilibrium points $\calE$. At each Monte Carlo run, we also obtain the closed-form solution $\MR^\star,\vt^\star$ using the method from~\cite{Horn87josa}.

We run 50 Monte Carlo runs, and we find that, (i) the set $\calE$ always contains four equilibrium points; (ii) the Jacobian $\calA$ at the four equilibrium points has \emph{zero, one, two and three} eigenvalues strictly on the right-half complex plane, respectively. Fig.~\ref{fig:eigenvalues} provides the superimposed scatter plots of the eigenvalues on the complex plane. This clearly shows that there are three equilibrium points that are \emph{locally unstable}; (iii) moreover, we compute the angular error between the rotation estimations at the four equilibrium points, denoted $\hat{\MR}$, and the closed-form rotation $\MR^\star$, which is $\scriptstyle \left|\arccos\left( \left( \trace{\hat{\MR}\tran \MR^\star} -1  \right) /2 \right)\right|$. We find that the three \emph{locally unstable} equilibrium points yield rotation errors that are always $180^\circ$, and the equilibrium point that has no eigenvalues with positive real parts always produces exactly $\MR^\star$,~\ie~rotation errors of $0^\circ$. These observations strongly support Conjecture~\ref{conj:characterizeEquilibria}. In addition, it suggests that the set of equilibrium points contains the globally optimal solution $\MR^\star$ and three other solutions that are generated from $\MR^\star$ by \emph{an additional rotation of $180^\circ$ around some axis}, and these solutions preserve torque balance.

\input{fig-eigenvalues}

\subsection{Simulation of the Dynamics}
We then simulate the dynamics in~\eqref{eq:dynamics} with constant time step size $0.01$s to solve 100 random instances of point cloud registration generated by the same procedure as in Section~\ref{sec:equilibria}. At each time step, we project the matrix $\MR^{(t+1)} = \MR^{(t)} + dt \cdot \dMR$ to $\SOthree$ to make sure it is a valid rotation matrix. We stop the simulation when $\|\dvs \| < 10^{-4}$. Fig.~\ref{fig:simulation}(a) shows the statistics of rotation errors over 100 Monte Carlo runs, demonstrating that the system always escapes the three locally unstable equilibrium points and converges to the globally optimal solution. Fig.~\ref{fig:simulation}(b) plots a typical example of the trajectories of the total energy $V$ and total potential energy $V_p$ \wrt time. We see that the total energy is always non-increasing, but the potential energy $V_p$ oscillates and eventually reaches its global minimum. From the lens of optimization, the simulation method can be seen as first building an upper bound (the total energy) for the target objective function (the potential energy), and then minimizing the upper bound. In this case, the upper bound eventually becomes tight.

\input{fig-simulation}

%% file: fig-eigenvalues.tex

\renewcommand{\mpwthree}{4cm}
\renewcommand{\myhspace}{\hspace{-3mm}}

\begin{figure}[t]
	\begin{center}
	\begin{minipage}{\textwidth}
	\begin{tabular}{cccc}%
		\hspace{-15mm}
			\begin{minipage}{\mpwthree}%
			\centering%
			\includegraphics[width=\columnwidth]{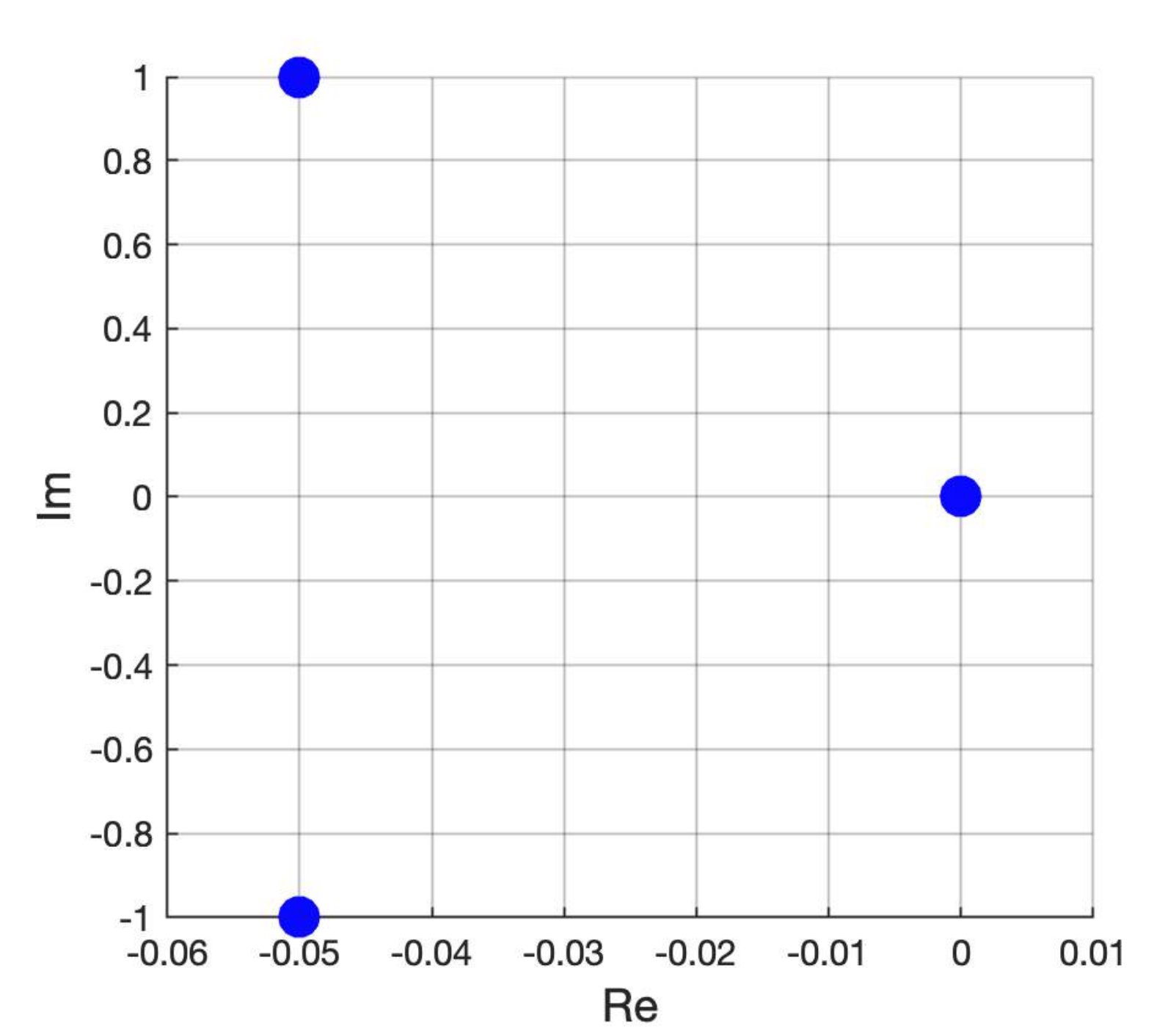} \\
			{\small (a) Zero eigenvalue with positive real parts.}
			\end{minipage}
		& \hspace{-5mm}
			\begin{minipage}{\mpwthree}%
			\centering%
			\includegraphics[width=\columnwidth]{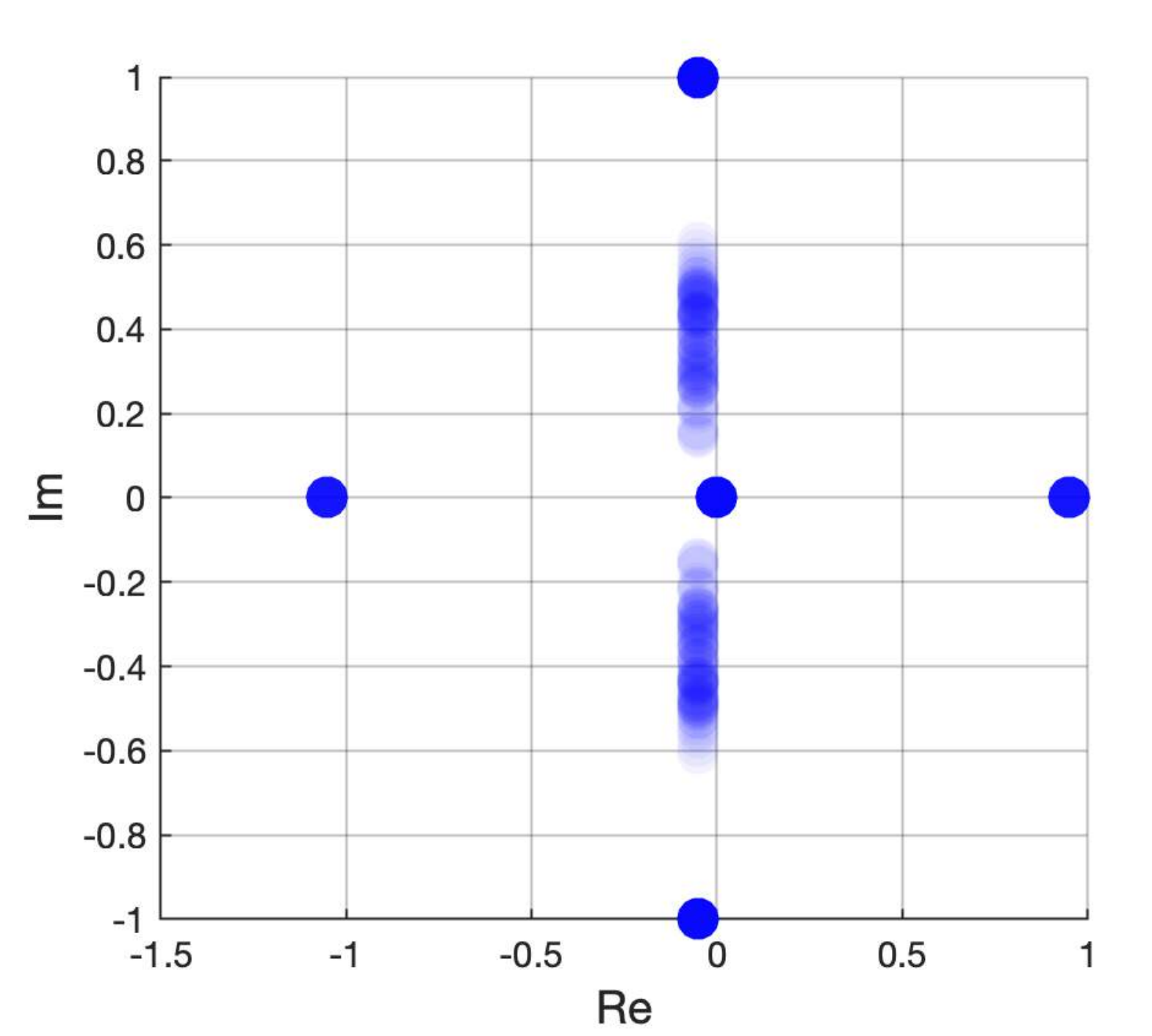} \\
			{\small (b) One eigenvalue with positive real parts.}
			\end{minipage}
		& \hspace{-5mm}
			\begin{minipage}{\mpwthree}%
			\centering%
			\includegraphics[width=\columnwidth]{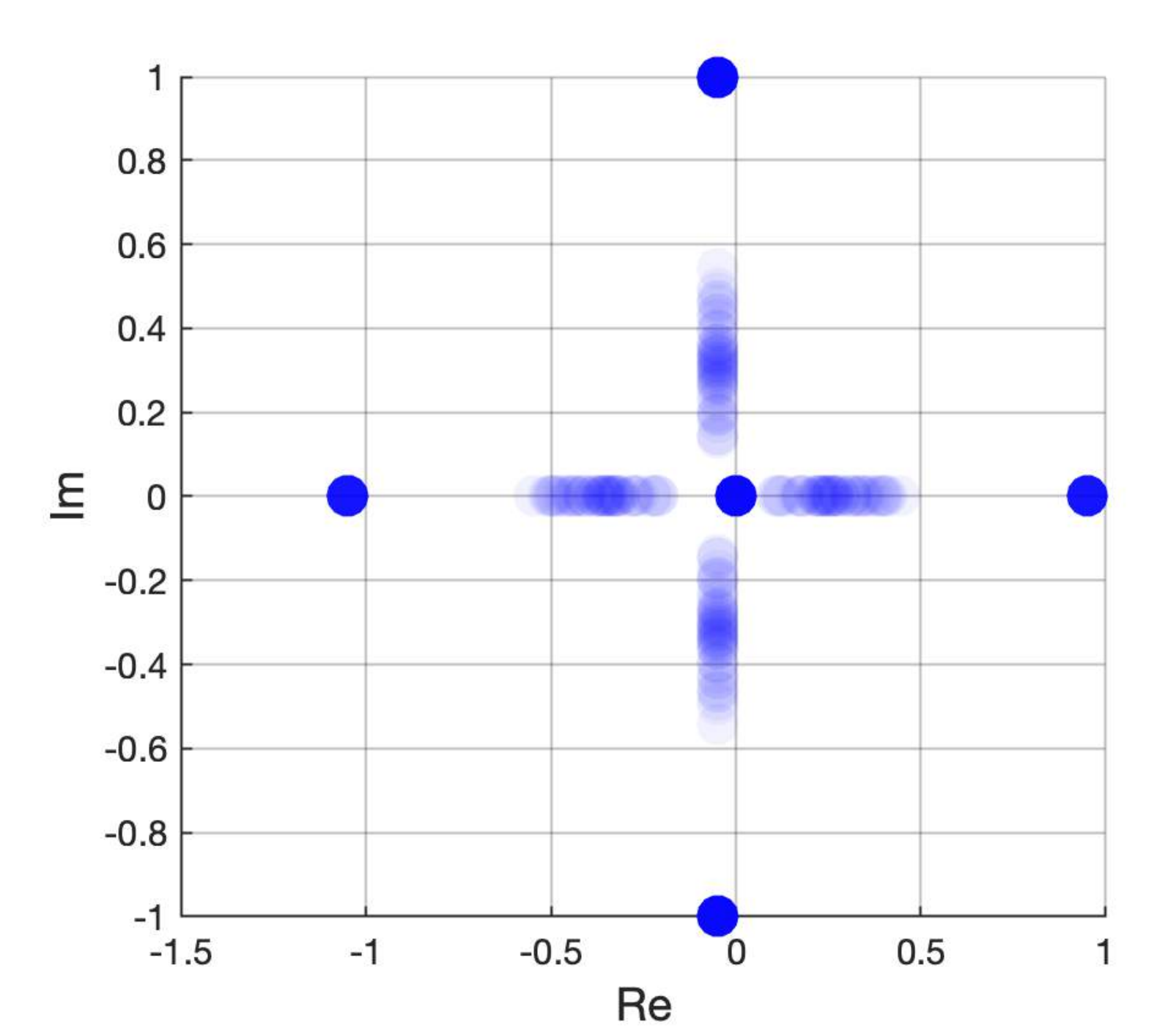} \\
			{\small (c) Two eigenvalues with positive real parts.}
			\end{minipage} 
		& \hspace{-5mm}
			\begin{minipage}{\mpwthree}%
			\centering%
			\includegraphics[width=\columnwidth]{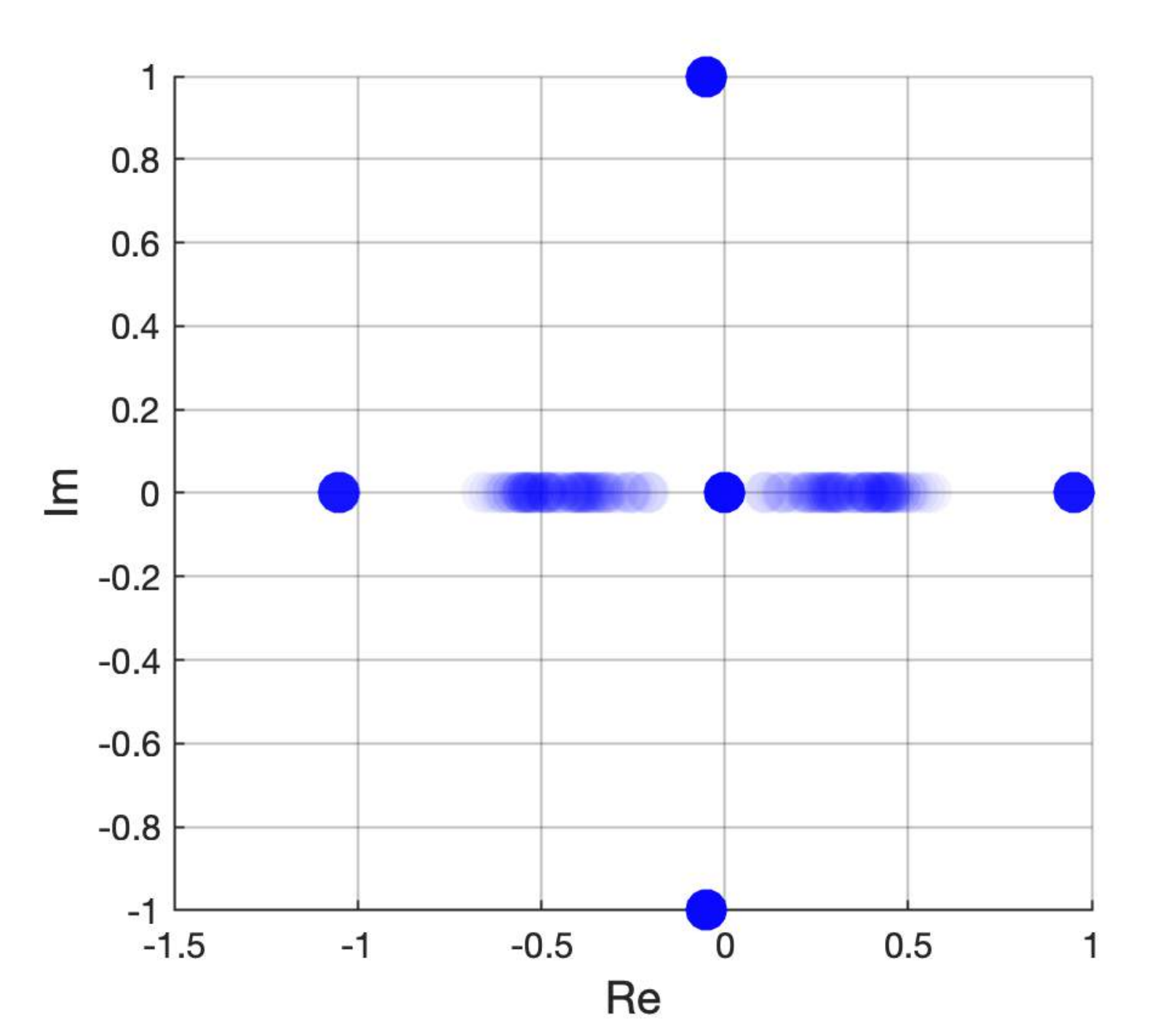} \\
			{\small (d) Three eigenvalues with positive real parts.}
			\end{minipage}
	\end{tabular}
	\end{minipage} 
	\caption{Eigenvalues of the Jacobian matrix $\MA$ on the complex plane, plotted from 50 Monte Carlo runs. (a) Eigenvalues of the Jacobian matrix $\MA$ at the \emph{globally optimal equilibrium}, all eigenvalues are on the left-half plane. (b) Eigenvalues of the Jacobian matrix $\MA$ at the first \emph{locally unstable equilibrium}, one eigenvalue lies strictly on the right-half plane. (c) Eigenvalues of the Jacobian matrix $\MA$ at the second \emph{locally unstable equilibrium}, two eigenvalues lie strictly on the right-half plane. (d) Eigenvalues of the Jacobian matrix $\MA$ at the third \emph{locally unstable equilibrium}, three eigenvalues lie strictly on the right-half plane. Each Monte Carlo run generates a transparent set of circles indicating the locations of the eigenvalues (opaque circles imply multiple occurrences of the same eigenvalue).
	\label{fig:eigenvalues}} 
	\vspace{-5mm} 
	\end{center}
\end{figure}

%% file: fig-simulation.tex

\renewcommand{\mpwthree}{6cm}
\renewcommand{\myhspace}{\hspace{-3mm}}

\begin{figure}[t]
	\begin{center}
	\begin{minipage}{\textwidth}
	\begin{tabular}{cc}%
			\begin{minipage}{\mpwthree}%
			\centering%
			\includegraphics[width=\columnwidth]{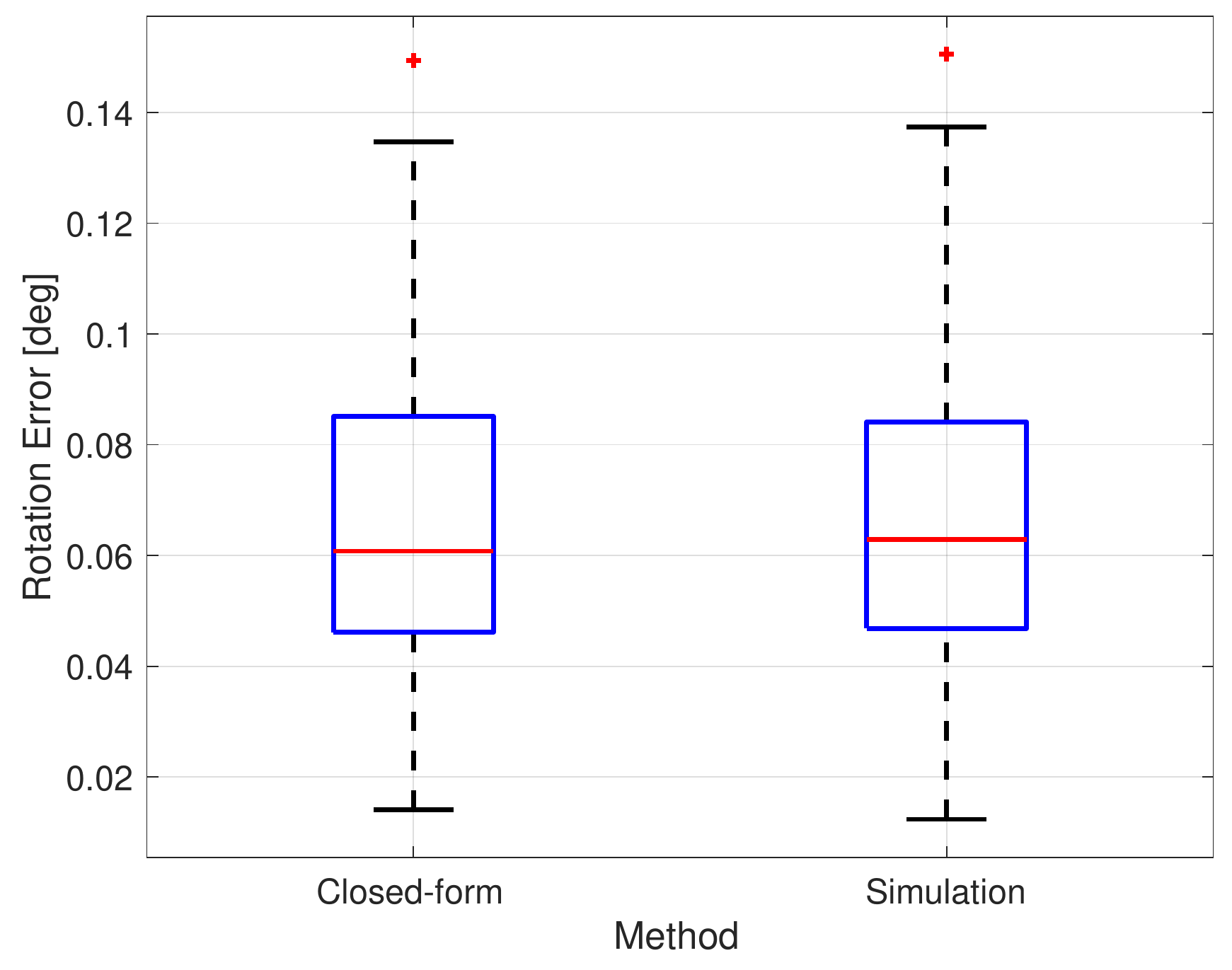} \\
			{\small (a) Rotation estimation error.}
			\end{minipage}
		& \hspace{10mm}
			\begin{minipage}{\mpwthree}%
			\centering%
			\includegraphics[width=\columnwidth]{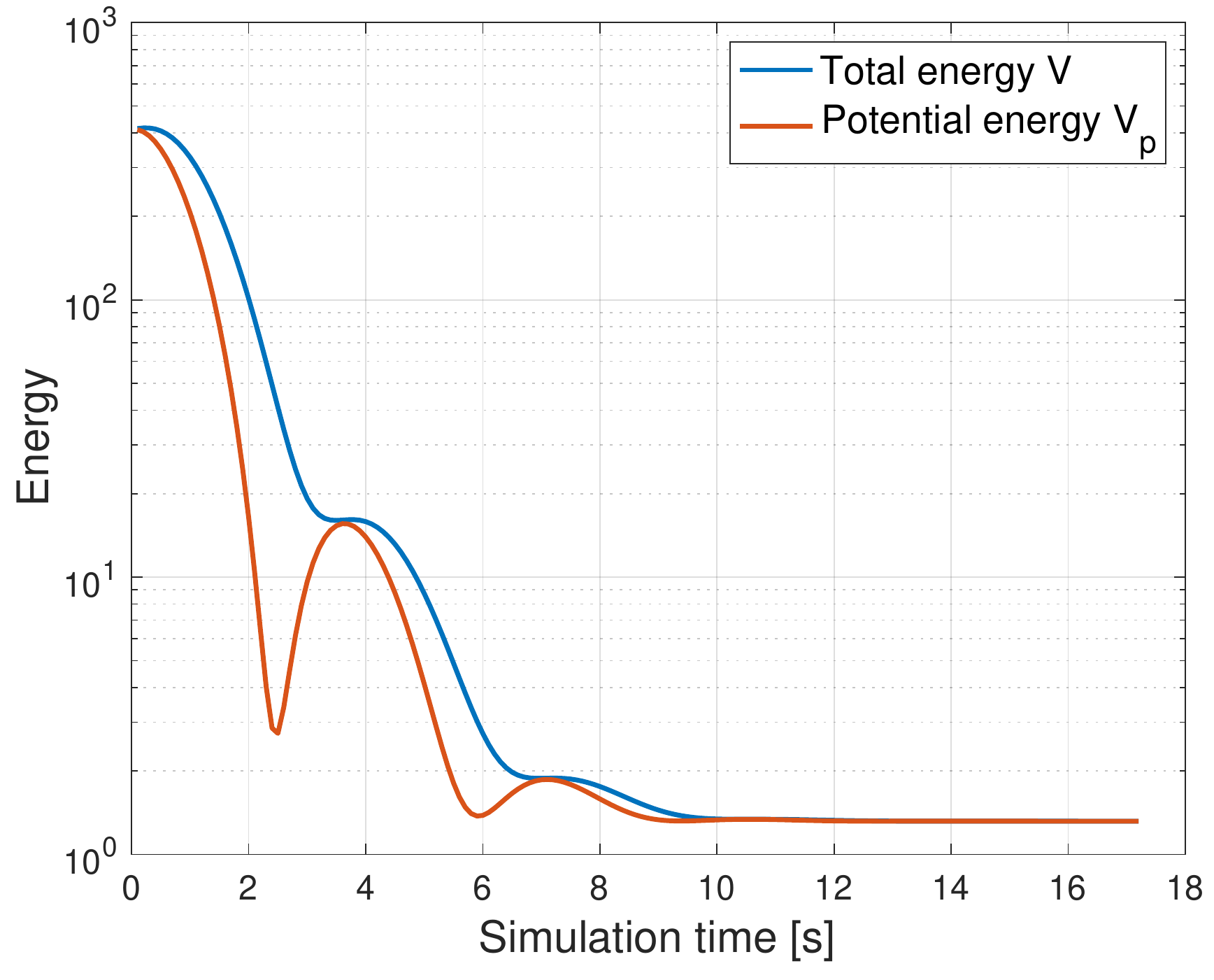} \\
			{\small (b) Convergence of $V$ and $V_p$.}
			\end{minipage}
	\end{tabular}
	\end{minipage} 
	\caption{Simulation of the dynamics. (a) Boxplot of rotation estimation errors over 100 Monte Carlo runs, using both the closed-form solution~\cite{Horn87josa} and the simulation of dynamics~\eqref{eq:dynamics}. (b) An example convergence plot of the total energy $V$ and the potential energy $V_p$ \wrt simulation time.
	\label{fig:simulation}} 
	\vspace{-5mm} 
	\end{center}
\end{figure}

%% file: conclusions.tex
\section{Conclusions}
\label{sec:conclusions}
In conclusion, we provide the first theoretical understanding of using dynamics-based simulation to solve point cloud registration. We treat point clouds as rigid bodies that contain a finite collection of particles and we treat registration as the process of one point cloud moving towards the other point cloud under forces and torques induced by virtual springs and viscous damping. We show that the total potential energy of the dynamical system recovers the objective function in maximum likelihood estimation. We then leverage the invariant set theorem to show that the system converges to the set of equilibrium points, of which one equilibrium point corresponds to the global optimizer of the maximum likelihood estimation. Further, supported by numerical experiments, we conjecture that, besides the globally optimal equilibrium point, the system contains either three or infinite spurious equilibrium points, and all spurious equilibria are locally unstable. This suggests that running the simulation can always obtain the global optimizer. Future work includes extending the analysis to the case where correspondences are corrupted by outliers.

%% file: acknowledgements.tex
\section*{Acknowledgments}
The authors would like to thank Prof. Jean-Jacques Slotine for fruitful discussions.